\documentclass[conference,10pt]{IEEEtran}

\usepackage[dvips]{graphicx}
\usepackage[english]{babel}
\usepackage[latin1]{inputenc}

\usepackage{amsmath,amsfonts,amssymb}
\usepackage{url}

\usepackage{cite} 
\usepackage{stfloats}  
                        
\usepackage{algorithmic,algorithm}

\def\atan{\mathop{\rm atan}}

\newcommand{\E} {\mathbf{E}}

\renewcommand{\P}{{\mathbb P}}

\newcommand{\N}{\mathbb{N}}

\newcommand{\bR}{\mathbb{R}}
\newcommand{\bP}{\mathbb{P}}

\newcommand{\bE}{\mathbb{E}}

\newcommand{\cA}{\mathcal{A}}

\newcommand{\cY}{\mathcal{Y}}
\newcommand{\cX}{\mathcal{X}}

\newtheorem{Proposition}{\bf Proposition}

\begin{document}

\title{Optimal Policies Search for Sensor
  Management : Application to the ESA Radar}
\author{
\authorblockN{\textbf{Thomas Br\'ehard}}
\authorblockA{R\&D Department\\
Huge Corporation\\
Gigantica, France}
\and
\authorblockN{\textbf{Emmanuel Duflos}}
\authorblockN{\textbf{Philippe Vanheeghe}}
\authorblockA{Ecole Centrale de lille\\
LAGIS UMR CNRS 8146\\
INRIA Lille - Nord Europe\\
Project Team SequeL\\
Cit\'e Scientifique - BP 46\\
59851, Villeneuve d'Ascq Cedex, France\\
Email: emmanuel.duflos@ec-lille.fr}
\and
\authorblockN{\textbf{Pierre-Arnaud Coquelin}}
\authorblockA{Predict and Control\\
INRIA Lille - Nord Europe\\
Parc Scientifique de la Haute Borne\\
40 Avenue Halley\\
59650, Villeneuve d'Ascq Cedex, France}
}

\maketitle

\selectlanguage{english}

\begin{abstract}
  This paper introduces a new approach to solve sensor management
  problems. Classically sensor management problems can be well
  formalized as Partially-Observed Markov Decision Processes
  (POMPD). The original approach developped here consists in deriving
  the optimal parameterized policy based on stochastic gradient
  estimation. We assume in this work that it is possible to learn the
  optimal policy off-line (in simulation ) using models of the
  environement and of the sensor(s). The learned policy can then be
  used to manage the sensor(s). In order to approximate the gradient
  in a stochastic context, we introduce a new method to approximate
  the gradient, based on Infinitesimal Approximation (IPA). The
  effectiveness of this general framework is illustrated by the
  managing of an Electronically Scanned Array Radar.
\end{abstract}


 \begin{center} \bfseries Sensor(s) Management Special Session \end{center}
%
\IEEEpeerreviewmaketitle
\section{Introduction}
\label{sec:Introduction}
Years after years the complexity and the performances of many sensors
have increased leading to more and more complex sensor(s)-based
systems which supply the decision centers with an increasing amount of
data. The number, the types and the agility of sensors along with the
increased quality of data far outstrip the ability of a human to
manage them: it is often difficult to compare how much information can
be gained by way of a given management scheme \cite{kalandros2004}. It
results from this the necessity to derive unmanned sensing platforms
that have the capacity to adapt to their environment
\cite{shihao2007}. This problem is often refered as the
\textit{Sensor(s) Management Problem}. In more simple situations, the
operational context may lead to works on sensor(s) management like in
the \textit{radar - infrared sensor} case \cite{blackman99}. A general
definition of this problem could then be : sensor management is the
effective use of available sensing and database capabilities to meet
the mission goals. Many applications deal with military applications,a
classical one being to detect, to track and tp identify smart targets
(a smart target can change its way of moving or its way of sensing
when it detects it is under analysis) with several sensors. The
questions are then the following at each time: how must we group the
sensors, how long, in which direction, and with which functioning
mode? The increasing complexity of the targets to be detected, tracked
and identified, makes the management even more difficult and led to
the development of researches on the definition of an optimal sensor
management scheme in which the targets and the sensors are treated
altogether in a complex dynamic system \cite{doucet2002}.

Sensor Management has become very popular this last years and many
approaches can be found in the litterature. In \cite{duflos2007a} and
\cite{duflos2007b} the authors use a the modelling of the detection
process of an Electronically Scanned Array (ESA) Radar to propose
management scheme during the detection step. In \cite{kastella1997,
  kolba2007, rostamabad2007} an information-based approach is use to
manage a set of sensors. From a theorical point of view the sensor
management can be modelled as a Partially Observable Markov Decision
Process (POMDP) \cite{kreucher2004, krish2002, krish2007}. Whatever
the underlying application, the sensor management problem consists in
choosing at each time $t$ an action $A_t$ within the set $\mathcal{A}$ of
available actions. The choice of $A_t$ is generally based on the
density state vector $X_t$ describing the environment of the system
and variables of the system itself. It is generally assumed that the
state or at least a part of this state is Markovian. Moreover in most
of the applications, we only have access to a partial information of
the state and $X_t$ must be estimated from the measurements
$\{Y_{s}\}_{1\leq s \leq t}$. This estimation process is often derived
within a Bayesian framework where we use state-dynamics and
observation models such as:

\begin{equation}
    X_{t+1}=F(X_t,A_t,N_t)
\end{equation}

\begin{equation}
    Y_t= H(X_t,W_t)
\end{equation}

where $N_t$, $W_t$, $F$ and $H$ respectively stands for the state
noise, the measurements noise, the state-dynamics and the measurement
function. $F$ and $H$ are generally time varying functions. The
control problem consists in finding the scheduling policy $\pi$
i.e. select $A_{t}$ given the past and the possible futures. However,
this control problem may have a theorical solution, it is generally
untractable in practice. However few works propose optimal solution in
the frame of POMDPs like \cite{krish2007}. Beside, several works have been
carried out to find sub-optimal policies like for instance myopic
policies. Reinforcement Learning and Q-Learning have also been used to
propose a solution (\cite{sutton90, kreucher2005}).

We propose in this paper to look for a policy within a class of
parametrized policy $\pi_{\theta}$ and to learn it which means learn
the optimal value of $\theta$. Funding our work on the approach
described in \cite{coquelin2008} we assume that it is possible to
learn this policy \textit{in simulation} using models of the overall
system. Once the optimal parameter has been found it is used to manage
the sensor(s). The frame of this work being the detection and
localization of targets, we show in the last part of this paper how it
may be applied the the management of an ESA radar.

The section \ref{sec:pomdp} described the modelling of a sensor
management problem using a POMDP approach. In the section
\ref{sec:policy} we derive the algorithm to learn the parameter of the
policy.In section \ref{sec:aesa} we show how this method may be used
for the tasking of an ESA radar. Finally section \ref{sec:simulations}
exhibits firts simulations results.

\section{Modelling}
\label{sec:pomdp}

\subsection{POMDP Modelling}

Let us consider three measurable continuous spaces denoted by $\cX$,
$\cA$ and $\cY$, $\cX$ is called the \textit{state space}, $\cY$ the
observation space and $\cA$ the \textit{action space}. We call
$\mathcal{M}(\cX)$ the set of all the measures defined on $\cX$. A
Partially-Observable Decision Proceess is defined by a \textit{state
  process} $(X_t)_{t\geq0} \in \cX$, an \textit{observation process}
$(Y_t)_{t\geq 1} \in \cY$ and a set of actions $(A_t)_{t\geq 1} \in
\cA$. In these definitions $t$ stands usually for the
\textit{time}. The state process is an homogeneous Markov chain with
initial probability measure $\mu(dx_0) \in \mathcal{M}(\cX)$ and with
Markov transition kernel $K(dx_{t+1}|x_t)$
(\cite{coquelin+deguest+munos2007}):

\begin{equation}
\label{eq:kernel}
\forall t \geq 0,X_{t+1}\sim K(\cdot|X_t)  
\end{equation}

\begin{equation}
X_0 \sim \mu 
\end{equation}

$(Y_t)_{t\geq 1}$ is called the observation is linked with the state
process by the conditional probability measure:

\begin{equation}
\mathcal{P}(Y_t\in dy_t|X_t=x_t)=g(x_t,y_t) \, dy_t
\end{equation}

where $g:\cX \times \cY \rightarrow [0,1]$ is the marginal density
function of $Y_t$ given $X_t$. In a general way, the state process
evolves continuously with respect to time $t$ whereas the observations
are made at sampled time $t_n$. A new observation is used to derive a
new action. We will therefore consider in the following the processes
$(X_t)_{t\geq0}$, $(Y_n)_{n\in\mathbb{N}}$, $(A_n)_{n\in\mathbb{N}}$
where $n$ stands for the index of the observation. We also assume that
there exists two generative functions $F_\mu:U\rightarrow\cX$ and
$F:\cX\times U\rightarrow\cX$, where $(U,\sigma(U),\nu)$ is a
probability space, such that for any measurable {\it test function }
$f$ defined over $\cX$ we have:

\begin{eqnarray}
\int_{\cX}f(x_t)K(dx_t|x_{t-1})&=&\int f(F(x_{t-1},u))\nu(du)
\end{eqnarray}

and

\begin{eqnarray}
\int_{\cX}f(x_0)\mu(dx_0)&=&\int f(F_\mu(u))\nu(du).
\end{eqnarray}

In many practical situations, $U=[0,1]^{n_U}$, and $u$ is a $n_U$-uple
of pseudo random numbers generated by a computer. For sake of
simplicity, we adopt the notations $K(dx_0|x_{-1})\triangleq\mu(dx_0)$
and $F(x_{-1},u)\triangleq F_\mu(u)$. Under this framework, the Markov
Chain $(X_t)_{t\geq0}$ is fully specified by the following dynamical
equation:

\begin{equation}
X_{t+1}=F(X_t,U_t),\;U_t\stackrel{i.i.d.}{\sim}\nu
\end{equation}

The observation process $(Y_n)_{n\in\N}$, defined on the measurable
space $(\cY,\sigma(Y))$, is actually linked with the state process by the
following conditional probability measure 

\begin{equation}
\bP(Y_n\in dy_n|X_{t_n}=x_{t_n},A_n)=g(y_n,x_{t_n},A_n)\lambda(dy_t)
\end{equation}

where $A_n\in\cA$ is defined on the measurable space $(\cA,\sigma(A))$
and $\lambda$ is a fixed probability measure on $(\cY,\sigma(Y))$. As
we assume that observations are conditionally independent given the
state process we cab write $\forall 1 \leq i,j\leq t,\;i\neq j$:

\begin{eqnarray}
\bP(Y_i \in dy_i, Y_j \in dy_j|X_{0:t},A_i,A_j)= \nonumber \\
\bP(Y_{i}\in dy_i|X_{0:t},A_i)\bP(Y_{j}\in dy_j|X_{0:t},A_j)
\end{eqnarray}

where we have adopted the usual notation $z_{i:j}=(z_k)_{i\leq k\leq j}$.

\subsection{Filtering distribution in a Partially-Observable Markov Decision Process}
Given a sequence of action $A_{1:n}$ and a sample trajectory of the
observation process $y_{1:n}$ and indices $\{n_1,n_2,t_1,t_2\}$ such
that $1\leq n_1\leq n_2 \leq n$ and $0\leq t_1\leq t_{n_1}\leq
t_{n_2}\leq t_2\leq t_n$, we define, using the , the
posterior probability distribution $M_{t_1:t_2|n_1:n_2}(dx_{t_1:t_2})$
by (\cite{DelMoral2004}):

\begin{eqnarray}
\label{equ:posterior}
\bP(X_{t_1:t_2}\in dx_{t_1:t_2}|Y_{n_1:n_2}=y_{n_1:n_2},A_{n_1:n_2})
\end{eqnarray}

Using the Feynman-Kac framework, the probabilty \ref{equ:posterior} can be written:

\begin{eqnarray}
\label{equ:posterior2}
\frac{\prod_{t=t_1}^{t_2} K(dx_t|x_{t-1}) \prod_{j=n_1}^{n_2}
G_{t_j}(x_{t_j})}{\int_{\cX^{t_2-t_1}}\prod_{t=t_1}^{t_2} K(dx_t|x_{t-1})
\prod_{j=n_1}^{n_2} G_{t_j}(x_{t_j})},
\end{eqnarray} 

where for simplicity's sake, $G_{t_n}(x_{t_n})\triangleq
g(y_n,x_{t_n},A_n)$ and $G_0(x_0)\triangleq0$. One of the main
interest here is to estimate the state at time $t$ from noisy
observations $y_{1:n_t}$ with $n_t$ the index of the last observation
just before time $t$. From a bayesian point of view this information
is completely contained in the so-called \textit{filtering
  distribution} $M_{t:t|1:n_t}$. In the following, the filtering
distribution will simply be denoted as $M_t$.

\subsection{Numerical methods for estimating the filtering distribution}
Given a measurable test function $f:\cX\rightarrow\bR$, we want to
evaluate
\begin{eqnarray}
M_t(f)=\bE[f(X_t)|Y_{1:n_t}=y_{1:n_t},A_{1:n_t}]
\end{eqnarray}

which is equal, using the Feynman Kac framework, to:

\begin{equation}
\frac{\bE[f(X_t)\prod_{j=1}^{n_t}G_{t_j}(X_{t_j})]}{\bE[\prod_{j=1}^{n_t}G_{t_j}(X_{t_j})]}
\end{equation}

In general, it is impossible to find $M_t(f)$ exactly except for
simple cases such as linear/gaussian (using Kalman filter) or for
finite state space Hidden Markov Models.  In the general dynamics,
continuous space case considered here, possible numerical methods for
computing $M_t(f)$ include the Extended Kalman filter, quantization
methods, Markov Chain Monte Carlo methods and Sequential Monte Carlo
methods (SMC), also called particle filtering. The basic SMC method,
called Bootstrap Filter, approximates $M_t(f)$ by an empirical
distribution $M_t^N(f)=\frac{1}{N}\sum_{i=1}^Nf(x_i^N)$ made of $N$
so-called \textit{particles} (\cite{Doucet+Godsill+Andrieu2000}). It
consists in a sequence of transition/selection steps: at time $t$,
given observation $y_t$ (\cite{coquelin2008}):

\begin{itemize}
\item {\bf Transition step:} (also called {\bf importance sampling} or
  {\bf mutation}) a successor particles population
  $\widetilde{x}^{1:N}_t$ is generated according to the state dynamics
  from the previous population $x_{t-1}^{1:N}$. The (importance
  sampling) weights $w^{1:N}_t = \frac{g(\widetilde{x}^{1:N}_t,y_t)}{\sum_{j=1}^N
    g(\widetilde{x}^j_t,y_t)}$ are evaluated.
\item {\bf Selection step:} Resample (with replacement) $N$ particles
  $x_t^{1:N}$ from the set $\widetilde{x}^{1:N}_t$ according to the
  weights $w^{1:N}_t$. We write
  $x_t^{1:N}=\widetilde{x}^{k^{1:N}_t}_t$ where $k^{1:N}_t$ are the
  selection indices.
\end{itemize}

Resampling is used to avoid the problem of degeneracy of the
algorithm, i.e. that most of the weights decreases to zero.  It
consists in selecting new particle positions such as to preserve a
consistency property :

\begin{equation}
\sum_{i=1}^N w^i_t \phi(\widetilde{x}^{i}_t) = \E[\frac{1}{N}\sum_{i=1}^N\phi(x^{i}_t)]
\end{equation}

The simplest version introduced in \cite{gordon1993} consists in
choosing the selection indices $k^{1:N}_t$ by an independent sampling
from the set $1\!:\!N$ according to a multinomial distribution with
parameters $w_t^{1:N}$, i.e. $\P(k_t^i = j)=w_t^j$, for all $1\leq
i\leq N$. The idea is to replicate the particles in proportion to
their weights. The reader can find some convergence results of
$M_t^N(f)$ to $M_t(f)$ (e.g.  Law of Large Numbers or Central Limit
Theorems) in \cite{DelMoral2004}, but for our purpose we note that
under weak conditions on the test function and on the HMM dynamics, we
have the asymptotic consistency property in probability, i.e.
$\lim_{N\rightarrow\infty}M_t^N(f)\stackrel{\bP}{=}M_t(f)$.

\section{Policy Learning Algorithm}
\label{sec:policy}

\subsection{Optimal Parameterized Policy for Partially-Observable Markov Decision Process}

Let $R_t$ be a real value reward function 
\begin{eqnarray}
R_{t} \triangleq R(X_t,M_{t}(f))\;. 
\end{eqnarray}
The goal is to find a policy 

\begin{equation}
\pi: \cA^n\times \cY^n\rightarrow \cA
\end{equation}

that maximizes the criterion performance :
\begin{eqnarray}
 J_{\pi} = \int_0^T \bE[R_{t}]dt
\end{eqnarray}

where $T$ is the duration of the scenario. Designing in practice
policies that depend on the whole trajectory of the past
observations/actions is unrealistic. It has been proved that the class
of stationary policies that depend on the filtering distribution
conditionally to past observations/actions $M_t$ contains the optimal
policy. In general the filtering distribution is an infinite
dimensional object, and it cannot be represented in a computer and so
is the policy. We therefore propose to look for the optimal policy in
a class of parameterized policies $(\pi_\alpha)_{\alpha\in\Gamma}$
that depend on a statistic of the filtering distribution :
\begin{eqnarray}
A_{n+1}=\pi_{\alpha}(M_{t_n}(f)) 
\end{eqnarray}

where $f$ is any test function. As the policy $\pi$ is parameterized by
$\alpha$, the performance criterion now depends only on $\alpha$. Thus we
can maximize it by achieving a stochastic gradient ascent with respect
to $\alpha$ :
\begin{eqnarray}
\alpha_{k+1} = \alpha_{k} +\eta_k\nabla J_{{\alpha_k}}, \quad k\geq 0
\end{eqnarray}

where $\nabla J_{{\alpha_k}}$ denotes the gradient of $J_{\alpha_k}$
w.r.t $\alpha_k$. By convention $\nabla J_{{\alpha_k}}$ is column
vector whose $i$-th component is the partial derivative with respect
to $\alpha_i$. $(\eta_k)_{k\geq 0}$ is a non-increasing positive
sequence tending to zero. We present in the two following subsection a
possible approach to estimate $\nabla J_{{\alpha_k}}$ based on
Infinitesimal Perturbation Analysis (IPA).

\subsection{Infinitesimal Perturbation Analysis for gradient estimation}
\label{ipa}

We assume that we can write the following equality at each $k$:

\begin{equation}
\nabla J_ {\alpha} = \int_0^T \nabla_{\alpha}\mathbb{E}[R_{t}]dt
\end{equation}

\begin{Proposition}
We have the following decomposition of the
gradient
\begin{eqnarray}\label{eq:ipa}
 \nabla_{\alpha}\bE[R_{t}] & =& 
\bE [M_{t}(fS_t)\nabla_{M_{t}(f)}R_{t}] \nonumber \\
&-&\bE [M_{t}(f)M_{t}(S_t)\nabla_{M_{t}(f)}R_{t}] \nonumber \\
&+& \bE[R_{t}S_{t}]
\end{eqnarray} 

where
\begin{eqnarray}\label{def:S}
 S_{t} = \sum_{j=1}^{p_t}\frac{\nabla_{\alpha}G_{t_j}(X_{t_j})}{G_{t_j}(X_{t_j})}
\end{eqnarray} 

\end{Proposition}

\begin{proof}
First let us rewrite $\nabla_{\alpha} \bE[R_{t}]$ as following:
\begin{eqnarray}\label{eq:appA:1}
 \nabla_{\alpha} \bE[R_{t}]  =\nabla_{\alpha}  \int_{\cX^{t}\times \cY^{n_t}} R_t U_tV_t\prod_{j=1}^{n_t}\lambda(dy_j)
\end{eqnarray}

where:

\begin{eqnarray}
\left\{
\begin{array}{cl}
U_t  &=   \prod_{i=0}^{t} K(dx_i|x_{i-1})\;,\\
V_t  &=   \prod_{j=1}^{n_t}G_{t_j}(x_{t_j})
\end{array}\right.\;.
\end{eqnarray}

Remarking that only $R_t$ and $V_t$ depends on $\alpha$ so that we obtain
\begin{eqnarray}\label{eq:appA:2}
\left\{
\begin{array}{cl}
 \nabla_{\alpha} V_t  &= S_tV_t \;, \\
 \nabla_{\alpha} R_t &= \nabla_{\alpha} M_t(f)\nabla_{ M_t(f)}R_t
\end{array}\right.
\end{eqnarray}
where $S_t$ is given by eq.(\ref{def:S}). Incorporating (\ref{eq:appA:1}) in (\ref{eq:appA:2}), we obtain
\begin{eqnarray}\label{eq:appA:3}
 \nabla_{\alpha}\bE[R_{t}]  = & \bE[\nabla_{\alpha} M_t(f)\nabla_{ M_t(f)}R_t]\;.
+ \bE[R_{t}S_{t}]\;.
\end{eqnarray}
Now using one more time (\ref{eq:appA:2}), we have
\begin{eqnarray}\label{eq:appA:4}
\nabla_{\alpha} M_t(f) &=& \nabla_{\alpha}\bE\Big[f(X_t)\frac{V_t}{\bE[V_t]} \Big]\nonumber\\
&=& \bE\Big[f(X_t)\frac{\nabla_{\alpha}V_t}{\bE[V_t]} \Big]-
\bE\Big[f(X_t)\frac{V_t\bE[\nabla_{\alpha}V_t]}{\bE[V_t]^2} \Big]\nonumber\\
&=& \bE\Big[f(X_t)S_t\frac{V_t}{\bE[V_t]} \Big]-
M_{t}S_t\bE\Big[\frac{V_t}{\bE[V_t]} \Big]\nonumber\\
&=& M_{t}(fS_t)-M_{t}(f)M_{t}(S_t)
\end{eqnarray}
so that we obtain (\ref{eq:ipa}) by incorporating (\ref{eq:appA:4}) in (\ref{eq:appA:3}).

\end{proof}

We can deduce directly Algorithm 1 from (\ref{eq:ipa}). It is
important to note that we must deal with two time-scales. This first
and the shorter one allows to simulate the continuous state $X_t$. On
the contrary the observation and action process are updated only each
time we get a new observation. These specific time is denoted $t_n$ in
the algorithm. That is the eason while there is an alternative to
update the variables $S_t$ and $\tilde{w}^{(i)}_{t-1}$. A new action
$A_n$ is also calculated each $t_n$ as already explained above. One
can also be surprised to calculate $R(X_t,M_t(f))$ using the sampled
value of $X_t$. To well understand this algorithm we must remind that
\textbf{the learning is made off-line} using a simulated process. It is
therefore possible to use the \textit{real} value of $X_t$ in this case.

\begin{algorithm}\caption{Policy Gradient in POMDP via IPA}
   \label{algo:ipa}
\begin{algorithmic}
   \STATE Initialize $\alpha_0 \in \Gamma$
   \FOR{$k=1$ {\bfseries to} $\infty$}%
   \FOR{$t=1$ {\bfseries to} $T$}
   \STATE Sample $u_t\sim\nu$ 
   \STATE Set $x_t = F(x_{t-1},u_{t})$,
   \STATE If $t = t_n$, sample $y_n\sim g(.,x_{t},a_n)\lambda(.)$
   \STATE Set ${s}_{t}=
\left\{\begin{array}{ll}s_{t-1}
+\frac{\frac{\partial g}{\partial\alpha}({x}_t,y_n,a_n)}{g({x}_t,y_n,a_n)}
&\textrm{ if } t=t_n\\
                       s_{t-1}&\textrm{ else }                               
     \end{array}\right.$ 
\STATE Set $\forall i\in\{1\ldots,I\}$ 

$\quad\tilde{x}_{t}^{(i)}=F(x_{t-1}^{(i)},a_{t-1},u_t^{(i)})$ where $u^{(i)}\stackrel{iid}{\sim}\nu$   

$\quad\tilde{s}^{(i)}_{t}=
\left\{\begin{array}{ll}s^{i}_{t-1}
+\frac{\frac{\partial g}{\partial\alpha}({x}^{i}_t,y_n,a_n)}{g({x}^{(i)}_t,y_n,a_n)}
&\textrm{ if } t=t_n\\
                       s^{i}_{t-1}&\textrm{ else }                               
     \end{array}\right.$

$\quad\tilde{w}^{(i)}_t = 
\left\{\begin{array}{ll}\frac{g({x}^{(i)}_t,y_n,a_n)\tilde{w}^{(i)}_{t-1}}{\sum_j g({x}^{(j)}_t,y_n,a_n)\tilde{w}^{(j)}_{t-1}}
&\textrm{ if } t=t_n\\
                       \tilde{w}^{(i)}_{t-1} &\textrm{ else }                               
     \end{array}\right.$
   \STATE Set
   $(x^{(i)}_t,s^{(i)}_{t})_{i\in\{1,\ldots,I\}}=(\tilde x^{(i)}_t,\tilde s^{(i)}_{t})_{i\in\{k_1,\ldots,k_I\}}$,  $k_{1:I}$
are selection indices associated to $(\tilde w^{(i)})_{i\in\{1,\ldots,I\}}$,
   \STATE $m_t(f)=\frac{1}{I}\sum_i f(x_{t}^{(i)})$, $m_t(s_t)=\frac{1}{I}\sum_i s^{(i)}$, $m_t(fs_t)=\frac{1}{I}\sum_i f(x_{t}^{(i)})s^{(i)}_{t}$,
   \STATE $a_{n+1} = \pi_{\alpha_k}(m_t)$ if $t=t_n$ 
   \STATE $r_t = R(x_t,m_t(f))$ 
   \STATE $\nabla r_t = (m_t(fs_t)- m_t(f) m_t(s_t))\frac{\partial R}{\partial m_t(f)}(x_t,m_t(f))+r_ts_t$
   \STATE $\nabla J_{\alpha_k} = \nabla J_{\alpha_k} +  \nabla r_t$
   \ENDFOR
   \STATE $\alpha_{k+1}= \alpha_k + \eta_k \nabla J_{\alpha_k}$
   \ENDFOR
\end{algorithmic}
\end{algorithm}

\section{Application to the ESA Radar}
\label{sec:aesa}
The ESA is an agile beam radar which means that it is able to point
its beam in any direction of the environnement almost instantaneously
without inertia. However, the targets in the environement are detected
w.r.t a probability of detection which depends on the direction of the
beam and the time of observation in this direction. In the following,
we precise first the nature of an action, then the influence of the
action onto the probability of detection and finally the nature of the
observations.

\paragraph*{Definition of the action}

The main property of an ESA is that it can point its beam without
mechanically adjusting the antenna. An ESA radar provides measurements
in a direction $\theta$. We note $\delta$, the time of observation in
this direction. In this work the the $n$-th action is :
 \begin{eqnarray}
 A_{n} = \begin{bmatrix}\theta_{n} & \delta_n \end{bmatrix}^T
\end{eqnarray}
with
\begin{eqnarray}
\left\{
\begin{array}{l}
\theta_{n}\in [-\frac{\pi}{2},\frac{\pi}{2}]\;,\\
\delta_n \in \mathbb R^+
\end{array}
\right.\quad \forall n\geq 0\;.
\end{eqnarray}

This is a simple possible action. One could increase the number of
componenets of an action by adding the emitted frequency for
instance. The action does not influence directly the observation
produced by the ESA but the probability of detection of a target.

\paragraph*{The probability of detection $P_d$}

It refers to the probability to detect a target and therefore to the
probability to obtain an estimation of the state of a target $p$ at
time $t_n$ denoted $X_{t_n,p}$ with action $A_{n}$. In this work,
$X_{t_n,p}$ is composed of the localisation and velocity components of
the target $p$ at time $t_n$ in the x-y plane:
\begin{eqnarray}
X_{t_n,p}=\begin{bmatrix}
rx_{t_n,p}&
ry_{t_n,p}& 
vx_{t_n,p}& 
ry_{t_n,p}
    \end{bmatrix}^T
\end{eqnarray}
where the subscript $T$ stands for \textit{matrix transpose}. The terms
$rx_{t_n,p}$ and $ry_{t_n,p}$ refers here to the position and
$vx_{t_n,p}$ and $vy_{t_n,p}$ the velocity of target $p$ at time
$t_n$. We also denote $D_{n,p}$ the random variable which takes values
$1$ if the radar produces a detection (and therefore an estimation)
for target $p$ and $0$ else :
\begin{eqnarray}
  D_n=\begin{bmatrix}
    D_{n,1} & \ldots & D_{n,P} 
    \end{bmatrix}^T\;.
\end{eqnarray}
As said previously, this probability also depends on the time of
observation $\delta_n$. Aerial targets being considering here, the
reflectivity of a target can be modelled using a Swerling I model
\cite{Curry2005}. We then have the following relation between the
probability of detection and the probability of false alarm $P_{fa}$
(i.e. the probability that the radar produce a detection knowing that 
there is no target) (\cite{Wintenby2003, duflos2007a}):

\begin{eqnarray}\label{def:probabilite detection}
 P_d(x_{t_n,p},A_{n})=
P_{fa}^{\frac{1}{1+\rho(x_{t_n,p},A_{n})}} 
\end{eqnarray}

where $\rho(x_{t_n,p},A_{n})$ is the target signal-to-noise ratio. In
the case of an ESA radar, it is equal to :

\begin{eqnarray}\label{def:SNR}
 \rho(x_{t_n,p},A_{n}) = 
\alpha \delta_n\frac{cos^2 \theta_{n}}{r_{t_n,p}^4}e^{-\frac{(\beta_{t_n,p}-\theta_{n})^2}{2B^2}}
\end{eqnarray}

where $r_{t_n,p}$ is the target range and $\beta_{t_n,p}$ the azimuth
associated to target $p$ at instant time $t_n$. $\alpha$ is a
coefficient which includes all the parameters of the sensor and $B$ is
the beamwidth of the radar. It is reminded in Appendix A how the
equations \ref{def:probabilite detection} and \ref{def:SNR} may be 
derived. If we make the assumption that all the
detections are independant, we can write :

\begin{eqnarray}
\bP(D_n = d_n|X_{t_n}=x_{t_n},A_n) = \nonumber \\
\prod_{p}^{P} \bP(D_{n,p}=d_{n,p}|X_{t_n,p}=x_{t_n,p},A_{n})
\end{eqnarray}
where
\begin{eqnarray}
 \bP(D_{n,p}=d_{n,p}|X_{t_n,p}=x_{t_n,p},A_{n}) = \nonumber \\
P_d(x_{t_n,p},A_{n})\delta_{d_{n,p}=1} + (1-P_d(x_{t_n,p},A_{n}))\delta_{d_{n,p}=0}
\end{eqnarray}

\paragraph*{Observation equation}

At instant time $t_n$, the radar produces a raw observation $Y_n$
composed of $P$ measurements :
\begin{eqnarray}
 Y_n=\begin{bmatrix}
    Y_{n,1} & \ldots & Y_{n,P} 
    \end{bmatrix}^T\;.
\end{eqnarray}
where $Y_{n,p}$ is the observation related to target of state value
$x_{t_n,p}$ obtained with action $A_{n}$ (we do not consider here the
problem of measurement-target association). Moveover, we assume that
the number of targets $P$ is known. Each of these measurements has the
following formulation :
\begin{eqnarray}
 Y_{n,p} =
\begin{bmatrix}
 r_{n,p}&
 \beta_{n,p}&
 \dot r_{n,p}
\end{bmatrix}^T
\end{eqnarray}
where $r_{n,p}$, $\beta_{n,p}$, $\dot r_{n,p}$ are range, azimuth and
range rate. The equation observation can be written
\begin{eqnarray}
\bP(Y_n\in dy_n|X_{t_n}=x_{t_n},A_n) = \\
\prod_{p}^{P} \bP(Y_{n,p}\in dy_{n,p}|X_{t_n,p}=x_{t_n,p},A_{n})
\end{eqnarray}
where
\begin{eqnarray}
\bP(Y_{n,p}\in dy_{n,p}|X_{t_n,p}&=&x_{t_n,p},A_{n})\\
&=&g(y_{n,p},x_{t_n,p},A_{n})\lambda(dy_{n,p})
\end{eqnarray}
\begin{eqnarray}
  g(y_{n,p},x_{t_n,p},A_{n}) = \begin{pmatrix}\mathcal{N}\left(h_t(x_{t_n,p}),\Sigma_y\right)P_d(x_{t_n,p},A_{n}) \\ 1-P_d(x_{t_n,p},A_{n})\end{pmatrix}^T  
\end{eqnarray}
and 
\begin{eqnarray}
\lambda(dy_{n,p}) = \lambda_{cont}(dy_{n,p}) + \lambda_{disc}(dy_{n,p}) 
\end{eqnarray}
The relation between the state and the raw observations is given by :

\begin{eqnarray}
  \label{eq:obs_equ}
  Y_{n,p}=h_{t_n}(X_{t_n,p})+W_{n,p}
\end{eqnarray}

with $h_{t_n}(x_{t_n,p})$ equals to:

\begin{eqnarray}
\begin{pmatrix}
\sqrt{(rx_{t_n,p}-rx^{obs}_{t_n})^2+(ry_{t_n,p}-ry^{obs}_{t_n})^2}\\
\atan\left\{\frac{ry_{t_n,p}-ry^{obs}_{t_n}}{rx_{t_n,p}-rx^{obs}_{t_n}}\right\}\\
\frac{(rx_{t_n,p}-rx^{obs}_{t_n})(vx_{t_n,p}-vx^{obs}_{t_n})+(ry_{t_n,p}-ry^{obs}_{t_n})(vy_{t_n,p}-vy^{obs}_{t_n})}
{\sqrt{(rx_{t_n,p}-rx^{obs}_{t_n})^2+(ry_{t_n,p}-ry^{obs}_{t_n})^2}}
 \end{pmatrix}
\end{eqnarray}
and $W_{n,p}$ a gaussian noise the covariance matrix of which is given by :
\begin{eqnarray}
 \Sigma_y = diag(\sigma^2_{r},\sigma^2_{\beta},\sigma^2_{\dot r})\;.
\end{eqnarray}

\paragraph*{State equation}
First let us introduce the definition of the unknown state $X_t$ at time $t$
and its evolution through time. $X_{t,p}$ is the state of the target
$p$. It has been defined above. Let $P$ be the known number of targets
in the space under analysis at time $t$. $X_{t}$ has the following form: .
\begin{eqnarray}
X_{t}=\begin{bmatrix}
    X_{t,1}&
    \ldots &
    X_{t,P} 
    \end{bmatrix}^T
\end{eqnarray}

Based on \cite{RongLi+Jilkov2003} works, we classically assume that all the 
targets follow a nearly constant velocity model. We use a
discretized version of this model (\cite{LeCadre+Tremois1998}) :
\begin{eqnarray}
 X_{t,p} = F(X_{t-1,p},U_{t}) \textrm{ where } U_t\sim\mathcal{N}\left(0,\sigma^2Q\right) 
\end{eqnarray}
where 
\begin{eqnarray}
 F = \begin{bmatrix}
      1 & 0 & \beta & 0 \\
      0 & 1 & 0 & \beta\\
      0 & 0 & 1 & 0 \\
      0 & 0 & 0 & 1
     \end{bmatrix}\text{ and }
Q = \begin{bmatrix}
      \frac{\beta^3}{3} & 0 & \frac{\beta^2}{2} & 0 \\
      0 & \frac{\beta^3}{3} & 0 & \frac{\beta^2}{2}   \\
      \frac{\beta^2}{2} & 0 & \beta & 0 \\
      0 & \frac{\beta^2}{2} & 0 & \beta
     \end{bmatrix}\;.
\end{eqnarray}

\section{Simulations}
\label{sec:simulations}

\section{Conclusion}

\section*{Appendix A}
We show in this Appendix how the probability of detection is
derived. First, the radar transmits a pulse expressed as follows
\begin{eqnarray}
 s(t) & = & \alpha(t)\cos(w_ct) \\
      & = &\text{Re}\{\alpha(t)e^{jw_ct}\}
\end{eqnarray}
where $\alpha(t)$ is the envelope also called the transmitted pulse
and $w_c$ the carrier frequency. This pulse is modified by the process
of reflection. A target is modelled as a set of elementary reflectors,
each reflecting: time delayed, Doppler shift, Phase shift and
attenuated version of the transmitted signal. We usually assume that
the reflection process is linear and frequency independent within the
bandwidth of the transmitted pulse. The return signal has the
following formulation:
\begin{eqnarray}
 s_r(t) =   G\sum_i\alpha(t-\tau_i)g_ie^{j(w_c(t-\tau_i+\frac{2\dot{r}_{i}}{c}t)+\theta_i)} +n(t)
\end{eqnarray}
where
\begin{itemize}
 \item $g_i$ is the radar cross section associated to reflector $i$,
 \item $\theta_i$ is the phase shift associated to reflector $i$,
 \item $\dot{r}_{i}$ is the radial velocity between the antenna and the object (Doppler frequency shift),
 \item $G$: others losses heavily range dependent due to spatial spreading of energy,
 \item $n(t)$ is a thermal noise of the receiver such that $\text{Re}\{n(t)\}, \text{Im}\{n(t)\} \sim \mathcal{N}(0,\sigma^2_{n})$.
\end{itemize}
We make the following approximations:
\begin{eqnarray}
\left\{\begin{array}{l}
 \dot{r}_{i} \approx \dot{r}\\
 \alpha(t-\tau_i)\approx \alpha(t-\tau)
\end{array}\right.
\end{eqnarray}
where $\dot{r}$ is the mean radial velocity of the target $\tau$ is
the mean time delay of the target. Using these approximations, the
return signal can be rewritten as follows:
\begin{eqnarray}
 s_r(t) =  \alpha(t-\tau)Ge^{j w_Dt}b +n(t) 
\end{eqnarray}
where 
\begin{eqnarray}
\left\{\begin{array}{cl}
w_D & =  w_c(1+\frac{2\dot{r}_{i}}{c}) \\
b  &= \sum_i g_ie^{j(-w_c\tau_i+\theta_i)}
\end{array}\right.\;.
\end{eqnarray}
The fluctuations of $b$ are known and modelled using Swerling 1 model
\cite{Curry2005}. There are differents models availables (Swerling 1,
2, 3,...) corresponding to different types of targets. Swerling 1
given below is convenient for aircrafts. We can then write :
\begin{eqnarray}
 \text{Re}\{b\}, \text{Im}\{b\} \sim \mathcal{N}(0,\sigma^2_{RCS})\;.
\end{eqnarray}
This modelling of $b$ assumes that the phase shifts $\theta_i$ are
independent and uniformly distributed and the magnitudes $g_i$ are
identically distributed. If the number of reflector is large, the
central limit theorem gives that $b$ is a complex-valued Gaussian
random variable centered at the origin. Now, a matching filter is
applied to our return signal
\begin{eqnarray}
 s_m(t) = \int_{-\infty}^{+\infty}s_r(t)h(s) ds
\end{eqnarray}
where $h(t)$ is a shifted, scaled and reversed copy of $s_r(t)$
\begin{eqnarray}
h(s) = \alpha(\delta-t)e^{-j w_D(\delta-t)}\;.
 \end{eqnarray}
 We choose $t=\delta+\tau$ which yields the best signal to noise ratio
 where $\delta$ is the length of the transmitted pulse. The
 probability of detection is based on quantity
 $|s_m(\delta+\tau)|^2$. We can show that
\begin{eqnarray}
 s_m(\delta+\tau) = Ge^{jw_D\tau}b + \int_{-\infty}^{+\infty} n(\delta+\tau-s)h(s)ds\;. 
\end{eqnarray}
One can remark that $s_m(\delta+\tau)$ is the sum of two complex-value
Gaussian variables.  We look at the following statistic
\begin{eqnarray}
 \Lambda = \frac{|s_m(\delta+\tau)|^2}{2\sigma_n^2}
\end{eqnarray}
and we introduce the following notation
\begin{eqnarray}
\sigma_s^2 = G^2 \sigma^2_{RCS}
\end{eqnarray}
Now we construct the test
\begin{eqnarray}
\left\{\begin{array}{l}
\mathcal{H}_1: \text{data generated by signal + noise}\\
\mathcal{H}_0: \text{data generated by noise}
\end{array}\right.
\end{eqnarray}

\begin{eqnarray}
\left\{\begin{array}{l}
\mathcal{H}_1: p_{\Lambda}(x)= \frac{1}{\frac{\sigma_s^2}{\sigma_n^2}+1}e^{-\frac{x}{\frac{\sigma_s^2}{\sigma_n^2}+1}}\\
\mathcal{H}_0: p_{\Lambda}(x)= e^{-x}
\end{array}\right.
\end{eqnarray}
Then, we derive the probability of detection and false alarm.
\begin{eqnarray}
\left\{\begin{array}{l}
P_d = \int_{\gamma}^{+\infty}p_{\Lambda}(x|\mathcal{H}_1\text{ is true})=e^{-\frac{\gamma}{\frac{\sigma_s^2}{\sigma_n^2}+1}}\\
P_{fa} = \int_{\gamma}^{+\infty}p_{\Lambda}(x|\mathcal{H}_0\text{ is true })=e^{-\gamma}
\end{array}\right.
\end{eqnarray}
Consequently
\begin{eqnarray}
 P_d=P_{fa}^{\frac{1}{\frac{\sigma_s^2}{\sigma_n^2}+1}}
\end{eqnarray}
The ratio $\frac{\sigma_s^2}{\sigma_n^2}$ is called the
Signal-to-Noise Ration noted $\rho$. This SNR is related to the
parameters of the system and the target. The classical radar equation
is given by the following formula (\cite{Wintenby2003}):

\begin{eqnarray}
 \rho = \frac{P_t G_t G_r \lambda^2 \sigma}{(4\pi)^3 r^4}
\end{eqnarray}

where $P_t$ is the energy of the transmitted pulse, $G_t$ is the gain
of the transmitted antenna, $G_r$ is the gain of the received antenna,
$\sigma$ is the radar cross section (for an aircraft between $0.1$ and
$1$ $m^2$), $r$ is the target range, $\gamma$ is the system noise
temperature and $L$ is a general loss term. However, the above formula
does not take into account for the sake of simplicity the losses due to
atmospheric attenuation and to the imperfection of the radar. Thus ,
extra terms must be added :
\begin{eqnarray}
 \rho = \frac{P_t G_t G_r \lambda^2 \sigma}{(4\pi)^3 k b L\gamma r^4}
\end{eqnarray}
where $b$ is the receiver noise bandwith (generally consider equal to the
signal bandwidth so that $b=\frac{1}{\delta_t}$), $k$ is Boltzmann's
constant, $\gamma$ is the temperature of the system and $L$ some
losses.  Moreover, the gain reduces with the deviation of the beam
from the antenna normal in an array antenna.
\begin{eqnarray}
 G_t & = G_0 cos^\alpha(\theta_t)\;,\\
 G_r & = G_0 cos^\alpha(\theta_t)
\end{eqnarray}
where $G_0$ is the gain of the antenna. In \cite{Vilmorin2002},
$\alpha = 2$, in \cite{Wintenby2003}, $\alpha = 2.7$.  According
\cite{VanKeuk+Blackman1993}, there is also a beam loss because the
radar beam is not pointing directly so that the radar equation is:
\begin{eqnarray}
 \rho = \frac{P_t G^2_0 \lambda^2 \sigma \delta_t \cos^2(\theta_t)}{(4\pi)^3 k  L\gamma r^4}
e^{-\frac{(\theta_t-\beta_t)^2}{2B^2}}
\end{eqnarray}
where is $B$ is the beamwidth.  

%

%


\bibliographystyle{IEEEtran}
\bibliography{MultiSensor,EduPva,m2vbib,biblio}

\begin{thebibliography}{10}
\providecommand{\url}[1]{#1}
\csname url@rmstyle\endcsname
\providecommand{\newblock}{\relax}
\providecommand{\bibinfo}[2]{#2}
\providecommand\BIBentrySTDinterwordspacing{\spaceskip=0pt\relax}
\providecommand\BIBentryALTinterwordstretchfactor{4}
\providecommand\BIBentryALTinterwordspacing{\spaceskip=\fontdimen2\font plus
\BIBentryALTinterwordstretchfactor\fontdimen3\font minus
  \fontdimen4\font\relax}
\providecommand\BIBforeignlanguage[2]{{%
\expandafter\ifx\csname l@#1\endcsname\relax
\typeout{** WARNING: IEEEtran.bst: No hyphenation pattern has been}%
\typeout{** loaded for the language `#1'. Using the pattern for}%
\typeout{** the default language instead.}%
\else
\language=\csname l@#1\endcsname
\fi
#2}}

\bibitem{kalandros2004}
\BIBentryALTinterwordspacing
M.~K. Kalandros, L.~Trailovi\'c, L.~Y. Pao, and Y.~Bar-Shalom, ``Tutorial on
  multisensor management and fusion algorithms for target tracking,'' in
  \emph{Proceeding of the 2004 American Control Conference Boston,
  Massachusetts June 30 - July 2}, 2004, pp. 4734--4748. [Online]. Available:
  \url{http://vehicle.me.berkeley.edu/~caveney/C3UV/papers/MSTrackingTutorialA%
CC04.pdf}
\BIBentrySTDinterwordspacing

\bibitem{shihao2007}
S.~Ji, R.~Parr, and L.~Carin, ``Nonmyopic multiaspect sensing with partially
  observable markov decision processes,'' \emph{IEEE Transactions on Signal
  Processing}, vol.~55, no.~6, pp. 2720--2730, June 2007.

\bibitem{blackman99}
S.~{Blackman} and R.~{Popoli}, \emph{Design and Analysis of Modern Tracking
  Systems}.\hskip 1em plus 0.5em minus 0.4em\relax Artech House Publishers,
  1999.

\bibitem{doucet2002}
A.~Doucet, B.~Vo, C.~Andrieu, and M.~Davy, ``Particle filtering for
  multi-target tracking and sensor management,'' \emph{Proceedings of ISIF},
  pp. 474--481, 2005.

\bibitem{duflos2007a}
E.~Duflos, M.~deVilmorin, and P.~Vanheeghe, ``Time allocation of a set of
  radars in a multitarget environment,'' in \emph{Proceedings of FUSION 2007
  Conference}, I.~S. on~Information~Fusion, Ed.\hskip 1em plus 0.5em minus
  0.4em\relax Quebec (Canada): International Society on Information Fusion,
  July 2007.

\bibitem{duflos2007b}
T.~Huguerre, E.~Duflos, T.~Br\'ehard, and P.~Vanheeghe, ``An optimal detection
  strategy for esa radars,'' in \emph{Proceedings of the COGnitive systems with
  Interactive Sensors Conference}, d.~l. e. d. T. d. l. e. d. l.~C. Société~de
  l'Electricité, Ed.\hskip 1em plus 0.5em minus 0.4em\relax Société de
  l'Electricité, de l'Electronique et des Technologies de l'Information et de
  la Communication, November 2007.

\bibitem{kastella1997}
K.~Kastella, ``Discrimination gain to optimize detection and classification,''
  \emph{IEEE Transaction on Systems, Man and Cybernetics - Part A : Systems and
  Human}, vol.~27, no.~1, pp. 112--116, January 1997.

\bibitem{kolba2007}
M.~Kolba and L.~Collins, ``Information based sensor management in the presence
  of uncertainty,'' \emph{IEEE Transactions on Signal Processing}, vol.~55,
  no.~6, pp. 2731--2735, June 2007.

\bibitem{rostamabad2007}
A.~Khodayari-Rostamabad and S.~Valaee, ``Information theoric enumeration and
  tracking of multiple sources,'' \emph{IEEE Transactions on Signal
  Processing}, vol.~55, no.~6, pp. 2765--2773, June 2007.

\bibitem{kreucher2004}
\BIBentryALTinterwordspacing
C.~Kreucher, D.~Blatt, A.~Hero, and K.~Kastella, ``Adaptive multi-modality
  sensor scheduling for detection and tracking of smart targets,'' in \emph{The
  2004 Defense Applications of Signal Processing Workshop (DASP), October 31 -
  November 5}, 2004. [Online]. Available:
  \url{http://www.eecs.umich.edu/~hero/Preprints/2004DASP.pdf}
\BIBentrySTDinterwordspacing

\bibitem{krish2002}
V.~Krishnamurthy, ``Algorithms for optimal scheduling and management of hidden
  markov model sensors,'' \emph{IEEE Transactions on Signal Procesing},
  vol.~50, no.~6, pp. 1382--1397, June 2002.

\bibitem{krish2007}
V.~Krishnamurthy and D.~Djonin, ``Structured threshold policies for dynamic
  sensor scheduling - a partially observed markov decision process approach,''
  \emph{IEEE Transactions on Signal Procesing}, vol.~55, no.~10, pp.
  4938--4957, October 2007.

\bibitem{sutton90}
R.~S. Sutton and A.~G. Barto, ``Time-derivative models of pavlovian
  reinforcement,'' \emph{Learning and Computational Neuroscience: Foundations
  of Adaptive Networks, M. Gabriel and J. Moore Eds.}, 1990.

\bibitem{kreucher2005}
\BIBentryALTinterwordspacing
C.~Kreucher and A.~Hero., ``Non-myopic approaches to scheduling agile sensors
  for multitarget detection, tracking, and identification,'' in \emph{The
  Proceedings of the 2005 IEEE Conference on Acoustics, Speech, and Signal
  Processing (ICASSP) Special Section on Advances in Waveform Agile Sensor
  Processing, volume V, March 18 - 23}, 2005, pp. 885--888. [Online].
  Available: \url{http://www.eecs.umich.edu/~hero/Preprints/2005ICASSP_a.pdf}
\BIBentrySTDinterwordspacing

\bibitem{coquelin2008}
P.~Coquelin and R.~Munos, ``Particle filter - based policy gradient in
  pomdps,'' February 2008, submitted at ICML'08.

\bibitem{coquelin+deguest+munos2007}
P.~Coquelin, R.~Deguest, and R.~Munos, ``Numerical methods for sensitivity
  analysis of feynman-kac models,'' INRIA-Futurs, Tech. Rep., 2007.

\bibitem{DelMoral2004}
P.~D. Moral, \emph{Feynman-Kac Formulae Genealogical and Interacting Particle
  Systems with Applications}.\hskip 1em plus 0.5em minus 0.4em\relax Springer,
  2004.

\bibitem{Doucet+Godsill+Andrieu2000}
A.~Doucet, S.~Godsill, and C.~Andrieu, ``On {S}equential {M}onte {C}arlo
  {S}ampling {M}ethods for {B}ayesian {F}iltering,'' Cambridge University
  Engineering Department, Tech. Rep., 2000.

\bibitem{gordon1993}
N.~Gordon, D.~Salmond, and A.~Smith, ``Novel approach to nonlinear and
  non-gaussian bayesian state,'' \emph{Proceedings IEE-F}, pp. 107--113, 1993.

\bibitem{Curry2005}
G.~Curry, \emph{Radar System Performance Modeling, Second Edition}.\hskip 1em
  plus 0.5em minus 0.4em\relax Artech House, 2005.

\bibitem{Wintenby2003}
J.~Wintenby, ``Resource allocation in airborn surveillance radar,'' Ph.D.
  dissertation, Chalmers University of Technology, 2003.

\bibitem{RongLi+Jilkov2003}
X.~{Rong Li} and V.~Jilkov, ``{A} {S}urvey of {M}aneuvering {T}arget {T}racking
  {P}art {I}: {D}ynamics {M}odels,'' \emph{{\sc ieee} Trans. on Aerospace and
  Electronic Systems}, vol.~39, no.~4, pp. 1333--1364, October 2003.

\bibitem{LeCadre+Tremois1998}
J.-P.~L. Cadre and O.~Tremois, ``Bearings-only tracking for maneuvering
  sources,'' \emph{{\sc ieee} Trans. on Aerospace and Electronic Systems},
  vol.~34, no.~1, pp. 179--193, January 1998.

\bibitem{Vilmorin2002}
M.~D. Vilmorin, ``Contribution \`a la grstion optimale de capteurs: application
  \`a la tenue de situations a\'eriennes,'' Ph.D. dissertation, Ecole Centrale
  de Lille et Universit\'e des Sciences et Technologie de Lille, 2002.

\bibitem{VanKeuk+Blackman1993}
G.~V. Keuk and S.~Blackman, ``On {P}hased-{A}rray {R}adar {T}racking and
  {P}arameter {C}ontrol,'' \emph{{\sc ieee} Trans. on Aerospace and Electronic
  Systems}, vol.~1, no.~29, pp. 186--194, January 1993.

\end{thebibliography}

\end{document}